\documentclass[lettersize,journal]{article}

\usepackage{fullpage}
\usepackage[sort&compress,numbers]{natbib}
\usepackage[utf8]{inputenc} %
\usepackage[T1]{fontenc}    %
\usepackage{hyperref}       %
\usepackage{url}            %
\usepackage{booktabs}       %
\usepackage{amsfonts}       %
\usepackage{microtype}      %
\usepackage{xcolor}         %
\usepackage{xspace}
\usepackage{algorithm}
\usepackage{algpseudocode}
\usepackage{graphicx}
\usepackage{amsmath,amssymb,amsthm}
\usepackage{authblk}
\usepackage{comment}
\newtheorem{theorem}{Theorem}[section] 
\newtheorem{lemma}{Lemma}     
\newtheorem{proposition}{Proposition} 
\newtheorem{corollary}{Corollary} 
\newtheorem{assumption}{Assumption}

\newcommand{\mypar}[1]{\noindent\textbf{#1}}
\newcommand{\mysubpar}[1]{\noindent\underline{\textit{#1}}}
\newcommand{\zexe}{ZEXE\xspace}
\newcommand{\gmma}{GAMMA\xspace}
\newcommand{\gbdt}{GBDT-EMBER\xspace}
\newcommand{\malconv}{MalConv\xspace}
\newcommand{\evasionrate}{ER\xspace}
\newcommand{\ozd}{OZD\xspace}
\newcommand{\random}{RS\xspace}

\newcommand{\progspace}{Z}
\newcommand{\manipspace}{\Delta}
\newcommand{\encoder}{\psi}
\newcommand{\decoder}{\xi}
\newcommand{\manipulation}{\rho}
\newcommand{\gradthreshold}{\omega}

\newcommand{\argmin}{\operatornamewithlimits{arg\,min}}

\newcommand{\scalarprod}[2]{\langle #1, #2 \rangle}

\usepackage{multirow}

\title{A New Formulation for Zeroth-Order Optimization of Adversarial EXEmples in Malware Detection}

\author[1]{Marco Rando}
\author[2]{Luca Demetrio}
\author[4]{Lorenzo Rosasco}
\author[2,3]{Fabio Roli}

\affil[1]{Universit\'e C\^ote d'Azur, Inria, CNRS, LJAD, Nice, France}
\affil[2]{University of Genova, Italy}
\affil[3]{University of Cagliari, Italy}
\affil[4]{MaLGa-DIBRIS,University of Genova, Italy}

\date{}

\begin{document}

\maketitle

\begingroup
\renewcommand\thefootnote{}\footnotetext{This is the preprint version of the following IEEE-published article: M. Rando, L. Demetrio, L. Rosasco and F. Roli, "A New Formulation for Zeroth-Order Optimization of Adversarial EXEmples in Malware Detection," in IEEE Transactions on Information Forensics and Security, vol. 21, pp. 506-515, 2026, doi: 10.1109/TIFS.2025.3648867.}
\endgroup

\begin{abstract}

Machine learning malware detectors are vulnerable to adversarial EXEmples, i.e., carefully-crafted Windows programs tailored to evade detection. Unlike other adversarial problems, attacks in this context must be functionality-preserving, a constraint that is challenging to address.  
As a consequence, heuristic algorithms are typically used, which inject new content, either randomly-picked or harvested from legitimate programs.
In this paper, we show how learning malware detectors can be cast within a zeroth-order optimization framework, which allows incorporating functionality-preserving manipulations. This permits the deployment of sound and efficient gradient-free optimization algorithms, which come with theoretical guarantees and allow for minimal hyper-parameters tuning.
As a by-product, we propose and study \zexe, a novel zeroth-order attack against Windows malware detection.  
Compared to state-of-the-art techniques, \zexe provides improvement in the evasion rate, reducing to less than one third the size of the injected content.
\end{abstract}

\section{Introduction}
\noindent Every week, 7M malicious software, known as \textit{malware}, are blocked by detection engines on VirusTotal.\footnote{\url{https://www.virustotal.com/gui/stats}}
While most of these are known samples that can be detected through pattern-matching with signatures, we assisted in the rise of machine-learning Windows malware detectors to deal with unknown variants that slip through signature-based detection~\cite{anderson2018ember, raff2018malware, trizna2022quo, gibert2020rise}.
However, recent works proved the existence of \emph{adversarial EXEmples}, carefully-manipulated Windows malware that bypass detection~\cite{demetrio2021adversarial}.
Differently from adversarial examples, which are computed by summing well-crafted noise to pixels of input images, these EXEmples must adhere to strict domain constraints imposed by the \textit{Windows PE file format},\footnote{\url{https://learn.microsoft.com/en-gb/windows/win32/debug/pe-format}}, which precisely specifies how programs are stored as files.
Hence, attackers must abandon additive manipulations in favor of \emph{practical manipulations} that perturb the structure of programs without causing corruption while maintaining the original functionality~\cite{demetrio2021functionality}.
As a result, the literature thrives with adversarial EXEmples attacks~\cite{demetrio2021adversarial, lucas2021malware, song2022mab, anderson2017evading, kolosnjaji2018adversarial}, with a strong focus on gradient-free techniques that inject new content, either randomly-picked or harvested from legitimate programs known as \textit{benignware}, to boost the likelihood of evasion~\cite{demetrio2021functionality}. While shown to be effective, the inclusion of practical manipulations increases the complexity of formally defining the optimization problem. Thus, none of the proposed gradient-free attacks provides theoretical guarantees %
since they are not supported by a formulation that permits such proofs. Also, although being the most powerful strategies also in \textit{black-box} scenarios where the target can only be reached via queries, current content-injection attacks~\cite{demetrio2021functionality, anderson2017evading, song2022mab} rely on unbounded perturbations, causing an uncontrolled growth in size of manipulated programs without the possibility of fixing it a priori.
Lastly, the choice of the injected content deeply influences the evasiveness, since it presents itself as a hyper-parameter of the optimization algorithm, thus requiring proper tuning.
Hence, in this paper, we overcome these limitations by proposing a novel formulation of adversarial EXEmples that defines practical manipulations as: (i) an initialization step that tackles the generation of fixed-size perturbations inside programs, and (ii) a rewriting step that overwrites the bytes selected by the initialization (\autoref{sec:formulation}).
With this trick, we can write a non-convex, non-smooth optimization problem directly in $\mathbb{R}^d$, while preserving the functionality in the original input space, being also solvable with algorithms backed by provable guarantees.
We then propose \zexe, a \textit{zeroth-order} optimization algorithm based on prior work~\cite{rando2024optimal} that leverages finite differences to estimate gradients and provides theoretical guarantees for non-convex, non-smooth functions (\autoref{sec:zexe}) along with an adaptive strategy to select the parameters, avoiding manual tuning.
We test \zexe against \gmma~\cite{demetrio2021functionality}, one of the most powerful content-injection attacks, showing that, with the same number of queries, our technique requires half of the perturbation size to increase the evasion rate against two state-of-the-art machine learning Windows malware detectors (\autoref{sec:experiments}).

\section{Background}
\label{sec:background}
\mypar{Adversarial EXEmples.} Machine learning models are vulnerable to \textit{adversarial examples}, carefully-crafted test-time samples that coerce the output of classification towards answers decided by attackers~\cite{biggio2018wild}.
Widely formulated on the image domain, recent work has also shown their efficacy in more complex scenarios like malware detection, where attackers create \textit{adversarial EXEmples}, Windows programs tailored to evade detection~\cite{demetrio2021functionality, demetrio2021adversarial} computed as the solution of a minimization problem:
\begin{equation}\label{eqn:general_adv_problem}
\delta^* \in \argmin_{\delta \in \manipspace} L(f(\manipulation( z;  \delta)), y)
\end{equation}
where $L$ is a loss function of choice that quantifies the error committed by a malware detector $f$ in assigning the class of legitimate programs $y$.
This error is computed on a malware sample $ z \in \progspace$ (where $\progspace \in [0,255]^*$ is the space of all programs stored as strings of bytes) perturbed with the \textit{practical manipulation} function $\manipulation : \progspace \times \manipspace \rightarrow \progspace$, which is functionality-preserving by design.
This means that input programs perturbed with $\manipulation$ are still functional without alteration of their execution, without requiring the inclusion of hard constraints in \autoref{eqn:general_adv_problem}.
We remark that these constraints are not investigated in the usual domain of images, since the careless perturbation of bytes inside programs causes corruption, rendering them unstable.
Lastly, the $\manipulation$ function injects a manipulation $ \delta \in \manipspace$, where $\manipspace$ is the space of the perturbation, which can be differently formalized~\cite{demetrio2021adversarial}.

\mypar{Structured Finite-Difference Method.}
To compute solutions of optimization problems without exploiting gradients, it is possible to rely on zeroth-order methods
that leverage \textit{finite-difference algorithms}~\cite{Kiefer1952StochasticEO,nesterov2017random}, which substitute the gradient of an objective function $f : \mathbb{R}^d \to \mathbb{R}$ with a surrogate $g$ built through finite differences over a set of random directions $(p^{(i)}){i = 1}^\ell$ with $\ell > 0$.
These directions can be either (i) sampled i.i.d. from a distribution ({\it unstructured methods}), or (ii) chosen through structural constraints like orthogonality ({\it structured methods}).
In particular, \cite{Berahas2022} observed that the latter provides a better approximation of the gradient compared to unstructured ones.
These rely on random orthogonal directions that can be represented as rotations of the first $\ell \leq d$ canonical basis vectors $(e_i){i = 1}^\ell$. Formally, we set $p^{(i)} = G e_i$ where $G$ belongs to the orthogonal group $O(d)$ defined as
\begin{equation*}
O(d) := { G \in \mathbb{R}^{d \times d} , | , \det G \neq 0\quad \wedge \quad G^{-1} = G^\intercal }.
\end{equation*}
Some methods to generate orthogonal matrices are described in \cite[Appendix D]{rando2024optimal}. Let $G \in O(d)$, $h > 0$ and $\ell \leq d$, we build a gradient surrogate as
\begin{equation}\label{eqn:struct_surrogate}
g_{(G, h)}(x) := \frac{d}{\ell} \sum\limits_{i = 1}^\ell \frac{f(x + h G e_i) - f(x - h G e_i)}{2h} G e_i,
\end{equation}
Starting from some initial guess $x_0 \in \mathbb{R}^d$, at every iteration $k \in \mathbb{N}$, the algorithm samples a random orthogonal matrix $G_k$ and computes the surrogate $g_{(G_k, h_k)}$. Then, it computes the new iterate $x_{k + 1}$ as
\begin{equation}\label{eqn:struct_zth}
x_{k + 1} = x_k - \gamma_k g_{(G_k, h_k)}(x_k),
\end{equation}
where $\gamma_k > 0$ is a stepsize. This algorithm has been analyzed in different settings \cite{rando2024optimal}. In particular, to perform the analysis in the non-smooth setting, a smoothing technique is exploited. Let $h > 0$, we define a smooth approximation of the objective function
\begin{equation}\label{eqn:smoothing}
f_h(x) := \frac{1}{\text{vol}(\mathbb{B}^d)} \int f(x + hu) du,
\end{equation}
where $\text{vol}(\mathbb{B}^d)$ denotes the volume of the unitary $\ell_2$ ball $\mathbb{B}^d := { x \in \mathbb{R}^d , | , |x|2 \leq 1}$. In particular, \cite{Bertsekas1973} showed that $f_h$ is differentiable even when $f$ is not. By \cite[Lemma 1]{rando2024optimal}, we have that $\mathbb{E}G[g{(G, h)}(x)] = \nabla f_h(x)$. This allows us to interpret the iteration \eqref{eqn:struct_zth} as Stochastic Gradient Descent on the smooth approximation of the objective function $F{h_k}$.

\section{Zeroth-Order Optimization of Adversarial EXEmples}
\label{sec:methodology}
\noindent In this section, we propose our novel formulation for adversarial EXEmples (\autoref{sec:formulation}), followed by our new zeroth-order optimization algorithm \zexe (\autoref{sec:zexe}) and its theoretical guarantees (\autoref{sec:theoretical}).

\subsection{Formulation of Adversarial EXEmples attacks}
\label{sec:formulation}
We propose a novel formulation for Adversarial EXEmples crafting problem that permits the application of well-established zeroth-order optimization algorithm. Such a definition addresses all those scenarios where the model to evaluate is either non-differentiable, or only accessible through scores only.
The idea consists of modifying the definition of the manipulation $\delta$ in \autoref{eqn:general_adv_problem}.
First, given an initial manipulation $ \delta \in \manipspace$, an target program $ z \in \progspace$ and a manipulation mapping $\manipulation : \progspace \times \manipspace \to \progspace$,  we can construct a manipulated program as
$\bar{ z} = \rho( z,  \delta)$.
This initial manipulation can be interpreted as the initialization of our attack. 
Note that the introduction of an initial manipulation allows us to view the program as composed of two parts: the effective program $ z$, which cannot be modified, and the manipulation $\delta$, which can be perturbed. 
To effectively adjust the manipulation, we define an encoding function $\encoder : \manipspace \to \mathbb{R}^d$ that maps byte values into the corresponding real numbers. Once the manipulation is encoded using this function, we can apply our perturbation just summing it to the encoded manipulation. Subsequently, to revert the perturbed data back to its original byte format, we define a decoder function $\decoder : \mathbb{R}^d \to \Delta$. This decoder function serves as the counterpart to the encoder, effectively mapping the perturbed real vector back to its byte representation. Finally, we apply the manipulation mapping to obtain a perturbed program. 
We, thus, formulate the problem of creating the perturbation to be applied to a program $z$ to compute an adversarial EXEmples as:
\begin{equation}\label{eqn:adv_problem}
    v^* \in \argmin\limits_{v \in \mathbb{R}^d} F_z(v) := L(f( \manipulation( z, \decoder(\encoder( \delta) + v) )), y).
\end{equation}
Note that with this formulation, we can use any zeroth-order optimization algorithm working on $\mathbb{R}^d$. 

\subsection{\zexe: Adaptive Zeroth-Order attack against Windows Malware Detectors}\label{sec:zexe}
In this section, we propose a zeroth-order algorithm to solve the problem formulated in \autoref{eqn:adv_problem}. 
Our approach extends the random search algorithm by merging random sampling with the finite-difference method. Given a sampling region $V \subseteq [0, 255]^d$, the algorithm samples a random initial guess ({\bf initialization}), it explores 'near' configurations ({\bf local exploration}), and refines such a guess with a finite-difference strategy ({\bf refinement}).

\mypar{Initialization.}
Given a program $z$, the algorithm begins by evaluating the objective function $F_z: R^d \to R$ at the initial manipulation $v_0 \in V$. The function value and the manipulation are stored as the current best evaluation and iteration found, respectively, i.e., 
\begin{equation*}
v_0^{\text{best}} = v_0 \qquad \text{and} \qquad F_0^{\text{best}} = F_z(v_0).    
\end{equation*}
\mypar{Local Exploration.} Then, at each iteration $k \in \mathbb{N}$, the algorithm samples a random orthogonal matrix $G_k$ and builds the following sets
\begin{equation*}
    \begin{aligned}
        V_k := \{(v_k + h_k G_k e_i) \}_{i = 1}^\ell \cup \{ v_k, v_k^{\text{best}}\} \quad \text{and} \quad \\D_k := \{F_z(v_k + h_k G_k e_i) \}_{i = 1}^\ell \cup \{ F_z(v_k), F_k^{\text{best}}\}.
    \end{aligned}
\end{equation*}
Note that this operation requires only $\ell + 1$ function evaluations since $F_{k}^{\text{best}} =F_z(v_k^{\text{best}})$. Then, we update the current best iterate and function value by comparing the current best value with the function values computed. Formally, 
\begin{equation}
    \begin{aligned}
        v_{k + 1}^{\text{best}} &\in \argmin\limits_{v \in V_k} F_z(v) \quad \text{and} \quad F_{k + 1}^{\text{best}} &=\min\limits_{F_k \in D_k} F_k.
    \end{aligned}\label{eqn:update_best}
\end{equation}

\mypar{Refinement.} Next, we compute the structured gradient surrogate $g_{(G_k,v_k)}(v_k)$ of the function $F_z$ using $D_k$ as
\begin{equation}\label{eqn:fw_struct_fd}
    g_{(G_k, h_k)}(v_k) := \frac{d}{\ell} \sum\limits_{i = 1}^\ell \frac{F_z(v_k + h_k G_k e_i) - F_z(v_k)}{h_k} G_k e_i.
\end{equation}
Note that our surrogate is different from \autoref{eqn:struct_surrogate}. Indeed, our approximation can be computed in $\ell + 1$ function evaluations instead of $2 \ell$. We decided to use this because we observed that it permits getting better performance in practice.  The manipulation is then updated according to the following iteration.
\begin{equation}\label{eqn:zexe_iteration}
v_{k + 1} = \left\{ \begin{array}{ll}
v_k - \gamma_k g_{(G_k, h_k)}(v_k) \qquad \text{if } \| g_{(G_k, h_k)}(v_k) \|^2 > \gradthreshold, \\
\text{sample } v \in V \subseteq [0, 255]^d \text{ at random} %
\end{array}
\right.
\end{equation}

Note that for $\gradthreshold = 0$, \autoref{algo:zexe} recovers the standard finite difference iteration analyzed in different works with different gradient approximations~ \cite{nesterov2017random,ghadimi_lan,duchi_power_of_two}. Choosing $\gradthreshold > 0$ allows for exploration behavior. Intuitively, if $g$ is a good approximation of the gradient of $F_z$, a small norm indicates that the algorithm is approaching a stationary point. To avoid getting stuck, we resample a new point, allowing us to explore a different region of the space. The algorithm can be summarized with the following pseudo-code.
\begin{algorithm}[H]
\caption{\zexe: Zeroth-order optimization attack}
\label{algo:zexe}
\begin{algorithmic}[1]
\Require malware $z$, initial manipulation $v_0$, a number of directions $\ell > 0$, a threshold $\gradthreshold \geq 0$, a sampling space $V \subseteq [0, 255]^d$
\State $v_0^{\text{best}} = v_0$
\State $F_0^{\text{best}} = F_z(v_0)$
\For{$k = 0, \cdots, T$}
    \State sample $G_k$ i.i.d. from $O(d)$
    \State compute $v_{k + 1}^{\text{best}}$ and $F_{k + 1}^{\text{best}}$ as in \autoref{eqn:update_best}
    \State $g_k = g_{(G_k, h_k)}(v_k)$ as in \autoref{eqn:fw_struct_fd}
    \State compute 
    \begin{equation*}
        v_{k + 1} = \left\{\begin{array}{ll}
            v_k - \gamma_k g_k & \text{if } \| g_k\|^2 \geq \gradthreshold  \\
            \text{sample } v \in V\text{ uniformly} & \text{else} 
        \end{array}        
        \right.        
    \end{equation*}

\EndFor
\State \textbf{return} $v_T^{\text{best}}$
\end{algorithmic}
\end{algorithm}

\subsection{High-dimensional Perturbations and Exploration Strategies.} 
The refinement step of \autoref{algo:zexe} alternates between moving in the opposite direction of the gradient approximation (exploitation) and performing a random search over a subset of bytes (exploration). As indicated in \autoref{eqn:zexe_iteration}, the decision of which step to perform depends on the squared norm of the gradient approximation. Specifically, at each iteration $k \in \mathbb{N}$, if $\|g_k\|^2$ is sufficiently small (i.e., $\|g_k\|^2 < \omega$), the algorithm performs an exploration step.
However, when the perturbation size is very large, this approach may become less effective, since a substantial number of function evaluations may be required to identify a suitable perturbation. In such cases, alternative strategies should be considered. To tackle such a setting, we propose two alternatives: \zexe-Langevin and \zexe-GA.

\mypar{\zexe-Langevin.} A possible alternative to random search is to merge exploration and exploitation by adding Gaussian noise to the iterate. For every $k \in \mathbb{N}$, let $\beta_k \geq 0$. We replace the iterate of \zexe (i.e., step $7$ of \autoref{algo:zexe}) with
\begin{equation}\label{eqn:zexe_lan_iterate}
v_{k + 1} = v_k - \gamma_k g_k + \sqrt{2\beta_k}\xi_k, \qquad \xi_k \sim \mathcal{N}(0, I).
\end{equation}
The intuition behind this update is that the gradient term encourages exploitation by moving the iterate along the approximated gradient direction, while the Gaussian perturbation introduces an exploration component that allows the algorithm to escape stationary points.
Notice that for $\beta_k = 0$, \autoref{eqn:zexe_lan_iterate} reduces to the finite-difference iteration as a special case, while for large $\beta_k$ the method approximates random search, since the noise term $\sqrt{2\beta_k}\xi_k$ dominates the update. 
Moreover, under suitable assumptions on the target function, this iterate (with a different gradient approximation) has been analyzed in \cite{zeroth_order_langevin}, providing theoretical guarantees on its convergence properties.

\noindent \mypar{\zexe-GA.} 
Another alternative is to employ a Genetic Algorithm (GA) \cite{Kramer2017} to perform exploration steps. The algorithm maintains a population of $p \in \mathbb{N}_+$ perturbations that are randomly initialized. Whenever $\|g_k\|^2 < \omega$, the elements of this population are sequentially used as candidate solutions and evaluated. For each candidate, the corresponding gradient approximation $g_k$ is computed. If the condition $\|g_k\|^2 < \omega$ still holds, the algorithm proceeds to the next candidate. Once all candidates in the population have been used, the population is updated through the standard GA operations of selection, mutation, and crossover. When a candidate satisfying $\|g_k\|^2 \geq \omega$ is found, it is used as the current iterate and updated according to \autoref{eqn:zexe_iteration}.

\subsection{Theoretical Results}
\label{sec:theoretical}

In this section, we explain the effect of the exploitation steps on the iteration by providing a theoretical analysis of \zexe. In particular, we provide convergence rates for non-convex, non-smooth targets. To do so, we consider a simplified setting in which the objective function is assumed to be Lipschitz continuous, i.e., it satisfies the following assumption.
\begin{assumption}[$L$-Lipschitz continuity]
\label{asm:lip_cont}
    The objective function $F_z$ is $L$-Lipschitz continuous i.e., there exists a constant $L > 0$ such that for every $x, y \in \mathbb{R}^d$,
    \begin{equation*}
        |F_z(x) - F_z(y)| \leq L \|x - y \|.
    \end{equation*}
\end{assumption}

Notice that such an assumption is not satisfied in real-world systems due to the discrete nature of the domain. However, we adopt it here to provide a formal explanation of the iterations of the algorithm and to build intuition about its behavior under controlled conditions. As we observe from \autoref{algo:zexe}, given a threshold $\gradthreshold \geq 0$, for every time-step $k \in \mathbb{N}$, if $\|g_{(G_k, h_k)(v_k)} \|^2 > \gradthreshold$, then the update of the iteration is computed as in structured finite-difference methods. Moreover, as in \cite{rando2024optimal}, for $h > 0$, $\mathbb{E}_{G}[g_{(G, h)}(\cdot)] = \nabla F_h(\cdot)$, where $g_{(G, h)}(\cdot)$ is defined in \autoref{eqn:fw_struct_fd} and $F_h(v) := \mathbb{E}_{u \sim \mathbb{B}^d}[F_z(v + hu)]$ - see \autoref{lem:smoothing_lemma} in \autoref{app:aux_results}. Thus, for every $0 \leq \underline{K} < \bar{K}$ such that for $k = \underline{K}, \ldots, \bar{K}$ the condition $\|g_{(G_k, h_k)(v_k)} \|^2 > \gradthreshold$ holds, we can provide rates for the expectation of the norm of the gradient of the smooth approximation of the objective function $F_h$. In the following, we refer to complexity as the number of function evaluations required to achieve an accuracy $\varepsilon \in (0, 1)$. We use the following notation, 
\begin{equation*}
\begin{aligned}
        \eta_{\underline{K}}^{\bar{K}} := \Bigg( \sum\limits_{i = \underline{K}}^{\bar{K}} \gamma_i \mathbb{E}[\| \nabla F_{h}(x_i) \|^2] \Bigg) / A_{\underline{K}}^{\bar{K}} \quad  \text{with} \quad A_{\underline{K}}^{\bar{K}} := \sum\limits_{i = \underline{K}}^{\bar{K}} \gamma_i.    
\end{aligned}    
\end{equation*}
\noindent In particular, we consider the setting in which the smoothing parameter is chosen as a constant $h_k = h$. Now, we state the main theorem for the non-convex non-smooth setting. Proofs are provided in the appendix.
\begin{theorem}[Rates for non-smooth non-convex functions]\label{thm:nonconv_eps_pos}
    Under \autoref{asm:lip_cont}, let $\gradthreshold > 0$. For every $0 \leq \underline{K} < \bar{K}$ such that for every $k =\underline{K}, \cdots, \bar{K}$, $\| g(v_k) \|^2 > \gradthreshold$, let $(v_k)_{k = \underline{K}}^{\bar{K}}$ be the sequence generated by \autoref{algo:zexe} from iteration $\underline{K}$ to $\bar{K}$. Then,
    \begin{equation*}
        \begin{aligned}
        \eta_{\underline{K}}^{\bar{K}} \leq S_{\underline{K}}^{\bar{K}} / A_{\underline{K}}^{\bar{K}} \qquad \\\text{with} \qquad S_{\underline{K}}^{\bar{K}} := F_{h_{\underline{K}}}(v_{\underline{K}}) - \min F_z + \frac{L^3 d^2\sqrt{d}}{\ell}\sum\limits_{i=\underline{K}}^{\bar{K}}\frac{\gamma_i^2}{h}.            
        \end{aligned}
    \end{equation*}
\end{theorem}
\noindent In the following corollary, we derive the rate for specific choices of the parameters.
\begin{corollary}\label{cor:nonconv_eps_pos}
    Under the same Assumptions of \autoref{thm:nonconv_eps_pos}, let $\gamma_k = \gamma$ with $\gamma, h > 0$. Then, for $k = \underline{K}, \cdots, \bar{K}$
        \begin{equation*}
            \eta_{\underline{K}}^{\bar{K}} \leq \frac{F_h(v_{\underline{K}}) - \min F_z}{\gamma (\bar{K} - \underline{K})} + \frac{L^3 d^2 \sqrt{d}}{\ell h} \gamma.
        \end{equation*}
    Moreover, let $K = \bar{K} - \underline{K}$ and $\varepsilon \in (0, 1)$. Then, if $K \geq 4 (F_h(v_{\underline{K}}) - \min F_z )(L^3 d^2 \sqrt{d} )\varepsilon^{-2}/(\ell h)$, choosing $\gamma = \sqrt{\frac{F_h(v_{\underline{K}}) - \min F_z )\ell h} {K L^3 d^2 \sqrt{d}}}$, we have that $\eta_{\underline{K}}^{\bar{K}} \leq \varepsilon$. Thus, the complexity of obtaining $\eta_{\underline{K}}^{\bar{K}} \leq \varepsilon$ is of the order $\mathcal{O}(d^2\sqrt{d}h^{-1}\varepsilon^{-2})$.
\end{corollary}

\mypar{Discussion.} In \autoref{thm:nonconv_eps_pos}, we provide a rate on the expected norm of the smoothed gradient.
The resulting bound consists of two parts. 
The first part is due to the (smoothed) functional value at the first point of the sequence for which the norm of the gradient approximation is larger than $\gradthreshold$ (i.e., the initialization). 
The second part is the approximation error. Recall that \autoref{asm:lip_cont} holds, and therefore we have that $F_{h}(v) \leq F_z(v) + L h$ for every $v \in \mathbb{R}^d$ - see \autoref{prop:smt_props}. Thus, by choosing $h$ small, we can reduce the gap between $F_h$ and $F$. However, choosing this parameter too small would increase the approximation error significantly. Note that with $\gradthreshold = 0$, \autoref{thm:nonconv_eps_pos} holds for every $k \in \mathbb{N}$. In particular, we obtain a rate that differs from \cite[Theorem 2]{rando2024optimal} only in the dependence on the dimension in the second term of $S_{\underline{K}}^{\bar{K}}$. We can achieve the same dependence on the dimension by choosing, for instance, the sequence $h = d \bar{h}$, where $\bar{h}$ is some non-negative sequence. 
In particular, we observe that if the number of consecutive iterations $k = \underline{K}, \ldots, \bar{K}$ for which $\|g_{(G_k, h_k)}(v_k)\|^2 > \gradthreshold$ is sufficiently large, we get that $\eta_{\underline{K}}^{\bar{K}} \leq \varepsilon$. This means that, under such a condition, in expectation, the iterations allow us to approach a stationary point of $F_h$.

\section{Experiments}\label{sec:experiments}
\noindent We evaluate the effectiveness of \zexe by comparing it against \gmma, a state-of-the-art attack that injects content harvested from benignware samples~\cite{demetrio2021functionality}. 

\noindent Our experimental study is conducted in two distinct settings: comparison with \gmma and high-dimensional perturbation. 

\subsection{Comparison with \gmma} 
\noindent In this first setting, we compare \zexe with random-search exploration against \gmma and other approaches using small perturbations.

\mypar{Baseline models.} We consider two state-of-the-art machine learning Windows malware detectors. The first one is \textit{\malconv}~\cite{raff2018malware}, an end-to-end deep convolutional neural network. \malconv takes in input programs as whole strings of bytes, cropping them if they exceed 2MB of input or, on the contrary, padding them with the special value 256. Due to the absence of any notion of distance in the domain of bytes, these are processed through an 8-dimensional embedding layer that makes them comparable. The second model is \textit{\gbdt}~\cite{anderson2018ember}, which is a gradient-boosting decision tree trained on a complex feature representation introduced to build the open-source EMBER dataset. In particular, EMBER features aggregate information regarding headers, imported APIs, byte histograms, readable strings, and more into 2381 features. For this paper, we leverage an open-source implementation of \malconv, with an input size shrunk to 1MB and detection threshold set to $0.5$, while for \gbdt we use its original open-source implementation\footnote{\url{https://github.com/elastic/ember}}, with detection threshold set at $0.8$. Both models are wrapped to be attacked through the SecML-Malware Python library~\cite{demetrio2021secml}. 

\mypar{Setup of \zexe.}
We set the number of explored directions $\ell = 5$ and the threshold on the gradient norm $\gradthreshold = 10^{-4}$. 
We set a number of queries of $1000$, thus the number of iterations of \zexe is $T = 1000 / (\ell + 1)$.
As practical manipulation, we use \textit{section injection}, which includes new content inside input Windows programs, like already-initialized strings, resources, or other assets~\cite{demetrio2021functionality}.
We initialize the manipulation with $d$ random printable bytes, ranging from the white-space character (encoded as 32) to the tilde character "\texttt{\~}" (encoded as 126). We define the encoding mapping $\encoder$ as the function that maps these bytes to their corresponding values in $\mathbb{R}^d$. The decoding function $\decoder$ projects the perturbed manipulation onto the hypercube $[32,126]^d \subseteq \mathbb{R}^d$. It then takes the integer part of each element in the perturbed manipulation and maps those integers back to the corresponding bytes, ensuring that the bytes remain within the range of printable characters. Finally, the manipulation function $\manipulation$ modifies the newly injected sections. We defined the sampling space $V$ as the hypercube $[32,126]^d \subseteq \mathbb{R}^d$. 
The stepsize $\gamma_k$ is choosen adaptively as 
\begin{equation*}
    \gamma_k =\max( \underline{\gamma}, \min (F_z(v_k) / \| g_{(G_k, h_k)}(v_k) \|^2, \bar{\gamma})),
\end{equation*}
where $0 < \underline{\gamma} \leq \bar{\gamma}$. Moreover, let $\underline{h} > 0$ and $h_0 \geq \underline{h}$. We define the following sequence of smoothing parameters
\begin{equation}\label{eqn:zexe_smt_params}
    h_{k + 1} = \left\{\begin{array}{ll}
        \max(h_k / 2, \underline{h}) & \text{if }\| g_{(G_k, h_k)}(v_k) \|^2 \geq \gradthreshold  \\
         h_0 & \text{otherwise} 
    \end{array}
    \right.
\end{equation} 
In our experiments, we fixed the number of sections to 50 and 100, with $2048$ and $4096$ bytes each for \malconv and \gbdt, resulting in an initial total manipulation amount of $102800$ and $410400$ bytes. 
It is important to note that additional bytes are needed to include sections inside programs due to format constraints that require content to span multiples of a specific value~\cite{demetrio2021functionality}.

\mypar{Setup of competing attacks.}
We compare \zexe different proposed attacks in the literature, by fixing the same number of queries to 1000.

\mysubpar{\gmma~\cite{demetrio2021functionality}:} this algorithm minimizes \autoref{eqn:general_adv_problem} with a generic algorithm leveraging the section injection manipulation.
\gmma can not fix a priori the size of the perturbation, but it uses a regularization parameter $\lambda$ to control such a quantity. Thus, we test different values of this parameter by choosing $\lambda \in [10^{-7}, 10^{-5}, 10^{-3}]$. Then we fix the number of sections to inject to $50$ and $100$ when computing the robustness of \malconv and \gbdt, respectively.

\mysubpar{\ozd~\cite{rando2024optimal}:} we include in the comparison the zeroth-order algorithm whose \zexe is based upon.
Thanks to our formulation, we can include this algorithm as well, even if it has not been proposed in the domain of Windows malware detection.
This method is a structured finite-difference algorithm that uses a gradient surrogate similar to the one in \autoref{eqn:fw_struct_fd}. 
However, it lacks the explorative procedure of our algorithm. In fact, this algorithm can be considered a special case of \autoref{algo:zexe} with $\gradthreshold < 0$, using the gradient surrogate defined in \autoref{eqn:struct_surrogate} and updating the best iterate and function values at any iteration $k$ as in \autoref{eqn:update_best}, using the sets $V_k := \{v_k, v_k^{\text{best}}\}$ and $D_k := \{F_z(v_k), F_k^{\text{best}}\}$. 
For a fair comparison, we used the same sets as \zexe for the update procedure.
Since \ozd also leverages finite differences, we set the number of random orthogonal directions used to approximate the gradient to $\ell = 5$. We manually tuned its smoothing parameter to $h_k = 100$ and its step size to $\gamma_k = \frac{1}{\sqrt{k}}$ for every iteration $k$.
For the encoder, decoder, perturbation, and the injected size, we use the same settings as \zexe for both models.

\mysubpar{Random Search (\random):} we include an algorithm that, using section injection, introduces the same number of sections as the other attacks, with the same perturbation size of \zexe and \ozd. However, the attack randomizes the included content at each query, standing as a baseline for our comparison. Notice that such a procedure can be recovered as a special case of \zexe when $\omega \to +\infty$.

\subsection{High-dimensional Perturbations} In the second setting, we consider high-dimensional perturbations. More precisely, we focus on comparison \zexe with different exploration strategies against \gmma using \gbdt as detector model. For \zexe, we set the number of explored directions $\ell = 5$. We set a number of queries of $500$. We use the same manipulation of the previous setting, and we initialize the manipulation using $d$ random bytes for each malware, without performing any initialization tuning.
We therefore defined the sampling space $V$ as the hypercube $[0, 256]^d \subseteq \mathbb{R}^d$. The stepsize $\gamma_k$ is chosen with a standard linesearch procedure \cite{zth_linesearch} and the smoothing parameter sequence is selected as in \autoref{eqn:zexe_smt_params} with $\underline{h} = 0.01$, $h_0 = 10^5$ and $\omega = 10^{-3}$. For every malware in the test set, we set the perturbation size as half of the average perturbation size achieved by \gmma (with $\lambda =0$) computed over $10$ repetitions. For \zexe-Langevin (\zexe-LAN) we set the sequence $\beta_k =  \frac{1}{2} \frac{\log(k + 2)}{k + 2}$, while for \zexe-GA we used a population size of $10$. The threshold on the gradient norm $\gradthreshold$ for \zexe with random search (\zexe-RS) and \zexe-GA is set to $10^{-2}$. The number of sections is set to $30$. For \gmma, we set the regularization parameter $\lambda =0$ and the number of sections to inject to $30, 40$, and $50$.

\mypar{Datasets.} 
To perform the experiments, we built two datasets of $200$ and $700$ malware, divided into different malware classes as shown in \autoref{tab:malware_classes} and \autoref{tab:malware_classes_2} extracted from the Speakeasy Dataset~\cite{trizna2022quo}.
We used the first dataset to perform the comparison with state-of-the-art methods and the second to explore the impact of the exploration strategies in high-dimensional setting.
We retrieve this information by querying VirusTotal,\footnote{\url{https:\\virustotal.com}} and by aggregating the results using AVClass~\cite{sebastian2020avclass2} to obtain the most-probable class for every single malware. Among these, we were unable to correctly classify 9 samples among 200 of the first dataset. 
For the second dataset, we obtained 700 malware samples by randomly selecting 100 samples from each of the malware families that constitute the Speakeasy dataset.
\begin{table*}[th]
\centering
\begin{tabular}{ccccccccc}
\textbf{grayware} & \textbf{virus} & \multicolumn{1}{l}{\textbf{backdoor}} & \textbf{downloader} & \textbf{worm} & \textbf{ransomware} & \textbf{miner} & \textbf{spyware} & \textbf{unknown} \\ \hline
59 & 36 & 15 & 51 & 20 & 8 & 1 & 1 & 9
\end{tabular}
\caption{Malware dataset used to test the robustness of Windows malware detectors.}
\label{tab:malware_classes}
\end{table*}

\begin{table*}[th]
\centering
\begin{tabular}{ccccccc}
\textbf{backdoor} & \textbf{coinminer} & \textbf{dropper} & \textbf{keylogger} & \textbf{ransomware} & \textbf{rat} & \textbf{trojan} \\ \hline
100 & 100 & 100 & 100 & 100 & 100 & 100 
\end{tabular}
\caption{Malware dataset used to compare the exploration strategies.}
\label{tab:malware_classes_2}
\end{table*}

\mypar{Evaluation Metrics.} For each attack, we report the evasion rate ($\evasionrate$), which is the fraction of computed adversarial EXEmples that bypass detection; the mean ($\mu_{\text{queries}}$) and standard deviation ($\sigma_{\text{queries}}$) of the number of queries required to achieve evasion, respectively; the mean ($\mu_{\text{size}}$) and standard deviation ($\sigma_{\text{size}}$) of the number of kilobytes (KB) or megabytes (MB) injected (i.e., the perturbation size); and the mean ($\mu_{\text{conf}}$) and standard deviation ($\sigma_{\text{conf}}$) of the confidence computed on adversarial EXEmples that successfully bypass the target model.

\mypar{Hardware setup.} All experiments have been conducted on a workstation equipped with an Intel\textsuperscript{\textregistered} Xeon\textsuperscript{\textregistered} Gold 5217 CPU with 32 cores running at 3.00 GHz, 192 GB of RAM, and two Nvidia\textsuperscript{\textregistered} RTX A6000 GPUs with 48 GB of memory each.

\subsection{Experimental Results}
\label{sec:results}
We first present the comparison between \zexe, \gmma, \ozd, and \random against the baseline models \malconv (\autoref{tab:malconv_queries}) and \gbdt (\autoref{tab:ember_queries}).

\begin{table}[ht]
    \centering
    \caption{Evasion rate achieved against \malconv with how many queries queries ($\mu_{\text{queries}}$, $\sigma_{\text{queries}}$) and size ($\mu_{\text{KB}}$ / $\sigma_{\text{KB}}$) computed on $200$ samples.}
    \label{tab:malconv_queries}
    \begin{tabular}{l c c c c c }
        \hline
        Algorithm & \evasionrate & $\mu_{\text{queries}}$ & $\sigma_{\text{queries}}$ & $\mu_{\text{size}}$ (KB) & $\sigma_{\text{size}}$ (KB)  \\
        \hline
        \random & $0.27$ & $124.96$ & $203.92$ & $211.537$ & $826.113$\\
        \gmma $[\lambda = 10^{-3}]$ & $0.14$ & $1.64$ & $1.26$ & $632.278$ & $829.358$\\
        \gmma $[\lambda = 10^{-5}]$ & $0.24$ & $12.80$ & $29.80$ & $648.923$ & $836.814$\\
        \gmma $[\lambda = 10^{-7}]$ & $0.58$ & $14.38$ & $29.24$ & $881.529$ & $860.288$\\
        \ozd & $0.55$ & $5.92$ & $15.07$ & $211.537$ & $826.113$\\ \hline
        {\bf \zexe} & {\bf 0.60} & $92.26$ & $158.92$ & {\bf 211.537} & {\bf 826.113}
    \end{tabular}
\end{table}
\noindent In \autoref{tab:malconv_queries}, we observe that \zexe achieves the highest evasion rate against \malconv. 
This indicates that the adversarial EXEmples crafted by our algorithm are more effective than those produced by the state-of-the-art approach, \gmma. Also, \zexe is able to achieve these results with a minimal amount of injected bytes: due to different structures of executable and different constraints, the mean perturbation achieved is 211KB, against the 881KB achieved by \gmma. Hence, \zexe only needs one fourth of the perturbation needed by \gmma.
However, we remark that \zexe requires a higher number of queries than the other algorithms to reach evasion.
This is mainly due to the fact that exploration is performed through a random search strategy, while the higher evasion rate is justified by the local exploration methods used. Note also that for \malconv, a simple finite-difference strategy as \ozd with our section injection obtains a high evasion rate with few queries required.

\begin{table}[ht]
    \centering
    \caption{
    Evasion rate achieved against \gbdt with how many queries queries ($\mu_{\text{queries}}$, $\sigma_{\text{queries}}$) and size ($\mu_{\text{KB}}$ / $\sigma_{\text{KB}}$) computed on $200$ samples.}
    \label{tab:ember_queries}
    \begin{tabular}{l c c c c c }
        \hline
        Algorithm & \evasionrate & $\mu_{\text{queries}}$ & $\sigma_{\text{queries}}$ & $\mu_{\text{size}}$ (KB) & $\sigma_{\text{size}}$ (KB)   \\
        \hline
        \random & $0.20$ & $60.32$ & $127.25$ & $475.705$ & $161.936$\\
        \gmma $[\lambda = 10^{-3}]$ & $0.08$ & $2.56$ & $4.12$ & $1261.796$ & $216.930$\\
        \gmma $[\lambda = 10^{-5}]$ & $0.12$ & $3.08$ & $6.26$ & $1258.171$ & $211.115$\\
        \gmma $[\lambda = 10^{-7}]$ & $0.20$ & $10.05$ & $18.52$ & $1513.767$ & $404.718$\\
        \ozd & $0.13$ & $12.59$ & $23.39$ & $475.705$ & $161.936$\\ \hline
        {\bf \zexe} & {\bf 0.37} & $187.38$ & $208.69$ & {\bf 475.705} & {\bf 161.936}
    \end{tabular}
\end{table}
In \autoref{tab:ember_queries}, we observe that \zexe once again achieves the highest evasion rate with a small average perturbation size against \gbdt. In this case, the difference between attacks is deepened by the difference in evasion rate (ER). In particular, \zexe computes 17\% successfully-evasive adversarial EXEmples more than the best competing version of \gmma. Regarding the manipulation size, \zexe only needs one third of the manipulation injected by \gmma.
Additionally, we note that \random matches the performance of \gmma, implying that (i) the choice of the initial benignware is crucial for the success of the attack; and (ii) that our choice of leveraging only printable strings inside the manipulation had boosted the effectiveness of \zexe. Lastly, \ozd performs worse against \gbdt than the other approaches. This is likely due to the lack of an exploration strategy and the non-convexity of the target function, which prevents the algorithm from finding good solutions. 

\begin{table}[H]
    \centering
    \caption{Average $\mu_{\text{conf}}$ and standard deviation $\sigma_{\text{conf}}$ of the confidence $F$ of \malconv and \gbdt computed on $200$ adversarial EXEmples.}\label{tab:malconv_ember_conf}
    \begin{tabular}{l cc cc}
        \hline
        \multirow{2}{*}{Algorithm} & \multicolumn{2}{c}{\malconv} & \multicolumn{2}{c}{\gbdt} \\
        \cline{2-5}
        & $\mu_{\text{conf}}$ & $\sigma_{\text{conf}}$ & $\mu_{\text{conf}}$ & $\sigma_{\text{conf}}$ \\
        \hline
        \random & $0.7344$ & $0.3850$ & $0.8577$ & $0.2695$ \\
        \gmma $[\lambda = 10^{-3}]$ & $0.8540$ & $0.2966$ & $0.9344$ & $0.1960$ \\
        \gmma $[\lambda = 10^{-5}]$ & $0.7614$ & $0.3635$ & $0.9001$ & $0.2538$ \\
        \gmma $[\lambda = 10^{-7}]$ & $0.4382$ & $0.4427$ & $0.8452$ & $0.3010$ \\
        \ozd & $0.4615$ & $0.4766$ & $0.8919$ & $0.2251$ \\ \hline
        {\bf \zexe} & \textbf{0.3995} & $0.4537$ & \textbf{0.7324} & $0.3339$ \\
    \end{tabular}
\end{table}

In \autoref{tab:malconv_ember_conf}, we report the mean and standard deviation of the confidence both \malconv and \gbdt attributed to the adversarial EXEmples crafted by the selected attacks. \zexe diminishes more the average confidence of both classifiers on all the adversarial EXEmples it crafts, w.r.t. all its competitors. While \gmma and \ozd manage to rival our attack on \malconv, where the difference is roughly 5\%, on \gbdt we observe a reduction of more than 11\% of the average confidence.

Finally, we consider the second setting where high-dimensional perturbations are used. In \autoref{tab:ember_exp2}, we report the evasion rate, the average, and standard deviation of the perturbation size of \gmma and variants of \zexe. As we can observe, due to the high-dimensionality, exploration strategies like random search and Langevin achieve worse performance than \gmma, suggesting that a more sophisticated method should be used. Instead, performing the exploration steps with genetic algorithm resulted to be effective, and it allowed for a higher evasion rate than \gmma even when a higher number of sections is used.

\begin{table}[ht]
    \centering
    \caption{Evasion rate of \gbdt achieved on $700$ malware.}
    \label{tab:ember_exp2}
    \begin{tabular}{l c c c }
        \hline
        Algorithm & \evasionrate &  $\mu_{\text{size}}$ (MB) & $\sigma_{\text{size}}$ (MB)   \\
        \hline
        \gmma ($30$ sections) & $0.4743$ & $22.394$ & $8.389$\\
        \gmma ($40$ sections) & $0.5354$ & $32.141$ & $11.172$\\
        \gmma ($50$ sections) & $0.5100$ & $36.214$ & $11.562$\\ \hline
        \zexe-RS & $0.2968$ & $11.113$ & $3.017$\\
        \zexe-LAN & $0.3271$ & $11.113$ & $3.017$\\
        {\bf \zexe-GA} & {\bf 0.5974} & {\bf 11.113} & {\bf 3.017}\\
    \end{tabular}
\end{table}

\section{Related Work}
\label{sec:related}
\mypar{Formulation of Adversarial EXEmples.} As already described in \autoref{sec:background}, \cite{demetrio2021functionality} proposes a formulation for computing adversarial EXEmples through the introduction of practical manipulations inside the optimization problem.
The authors also show that most of the proposed attacks can be recast inside their proposal, highlighting its generality.
\cite{pierazzi2020intriguing} propose a general formulation for adversarial evasion attacks that can be customized to deal with Windows programs as well.
To do so, the authors include an ``oracle'' function that informs whether manipulations maintain the original functionality or not.
Also, while we do not need any ``oracle'' since we deal with manipulations which preserve functionality by design,  \cite{pierazzi2020intriguing} introduces such a function, but they do not explain how to practically apply it inside the optimization problem, rendering it unapplicable.

\mypar{Gradient-free Adversarial EXEmples attacks.}
The state of the art provides two different approaches to create adversarial EXEmples: (i) algorithms that optimize the sequence of manipulations~\cite{anderson2017evading, song2022mab}; and (i) genetic algorithms~\cite{demetrio2021functionality, demetrio2021adversarial} that fine-tune injected content inside samples.
Among them, \cite{demetrio2021functionality} proposes \gmma, which leverages section injection to include benignware-extracted content inside malware, controlling the perturbation size with a  regularization parameter. However, all these techniques lack theoretical guarantees, binding the results of attacks to mostly-random heuristics with suboptimal performances even for synthetic optimization problems or in controlled scenarios, without any theoretical support to describe their behavior.
Also, these attacks either require resource-demanding sandbox validation~\cite{castro2019aimed,song2022mab}, or they can not fix the perturbation size a priori due to the nature of the chosen practical manipulation~\cite{anderson2017evading, demetrio2021functionality}.
Lastly, \gmma, one of the most-used attacks, strictly depends on a distribution of a pre-selected distribution of benignware, thus heavily influencing the evasion rate.
On the contrary, our formulation of zeroth-order attacks covers all these issues.
On top of this, we prove theoretical guarantees of \zexe to support and describe the algorithm behavior under classical optimization hypotheses, and we show that it provides better performance
in terms of trade-off between queries and perturbation with respect to other competing approaches. 

\mypar{Zeroth-order Optimization Algorithms.} In the literature, many methods were proposed to tackle the optimization problem using only function values, and finite-difference methods have been widely studied and analyzed in different settings with different gradient approximations \cite{nesterov2017random,duchi_power_of_two,shamir2017optimal,Berahas2022,ghadimi_lan,rando2022stochastic,rando2024optimal,Kozak2023,zo_bcd,zoro}. 
Some of these~\cite{zoro,zo_bcd} have also been applied to compute adversarial examples of image and audio samples as well. 
\cite{zoro} introduces a finite-difference method that builds a linear interpolation gradient estimator~\cite{Berahas2022} and exploits sparsity by including constraints on the approximation. However, due to the domain constraints imposed by the Windows PE file format, this method cannot be used for solving this problem directly, and, due to the fixed perturbation size applied by practical manipulation, it is not evident that this algorithm can exploit sparsity in this domain. On the contrary, our formalization allows us to use these methods as we did for \ozd in \autoref{sec:experiments}. Moreover, both empirical and theoretical evidence suggest that \zexe, our new optimization algorithm, performs better than unstructured methods~\cite{Berahas2022,choromanski} in controlled scenarios. Also, \zexe avoids stagnation in stationary points thanks to the exploration we included, differently from other finite-difference methods, and we empirically see the effect of this improvement in \autoref{sec:results}, achieving a higher evasion rate than the structured zeroth-order method \ozd.

\section{Contributions, Limitations, and Future Work}
\label{sec:conclusions}
In this work, we have proposed a novel formulation for computing adversarial EXEmples that can enable the application of sound and efficient gradient-free optimization algorithms even in the presence of complex domain constraints.
Thus, we define \zexe, a zeroth-order technique that allows for better performance in terms of the evasion rate and size of the injected content.

However, we only investigated the effect of one practical manipulation $\manipulation$, while more have been proposed~\cite{demetrio2021functionality}. Nevertheless, our formulation is general enough to tackle all of them by repeating the same experiments with different implementations of $\manipulation$. We tested \zexe against \gmma and \ozd, removing other strategies from the comparison. While this can be accomplished through more experiments, we want to remark that \cite{song2022mab, castro2019aimed, anderson2017evading} need either samples to be validated inside an isolated environment, or training complex reinforcement learning policies, thus increasing the complexity of their setup, which is absent for \zexe.

Also, \zexe needs, on average, more queries than its competitors to compute successfully-evasive adversarial EXEmples, thus being potentially detectable. Such is caused by both the exploration strategy (random search requires many iterations before sampling a good perturbation due to the dimensionality of the perturbation space) and the finite-difference approximation of the gradient. 
However, this requirement could potentially be reduced in future work by adopting more sophisticated exploration strategies and by also approximating second-order information.

Moreover, we did not assess the impact of the initialization of adversarial perturbation in terms of convergence and stability of \zexe (i.e., initialize all values with zero, or content harvested from benign software~\cite{demetrio2021functionality}). 
Although we limited our analysis to the usage of printable characters, we observed in \autoref{sec:results} that \zexe still achieves the best trade-off between evasion rate and perturbation size with respect to the considered approaches.

We want to remark that our work, even if focused on evasion attacks, has the purpose of expanding the robustness testing techniques in this domain, which present inherent complexities that must be tackled accordingly. Thus, we expect our work to be used not as an offensive strategy, but as a testing technique applied before placing detectors in production to avoid the spreading of weak detectors against adversarial EXEmples.

We conclude by envisioning that our new formulation will potentially enable more investigation of sound and theoretically-grounded optimization algorithms in the domain of adversarial EXEmples. Due to the ease of application and hyper-parameter tuning, we envision \zexe to be used to benchmark the robustness of Windows malware detectors with different manipulation functions.
This would produce a robustness leader board, similarly to what has been done in the image classification domain~\cite{croce2021robustbench}, which can be used to fairly compare the performance of all the proposed detectors in the literature, in particular the ones developed to resist adversarial EXEmples attacks~\cite{lucas2023adversarial, huang2024rs, gibert2023certified}.
Such would require a massive experimental analysis comprising more detectors and larger test datasets, providing a complete and comprehensive comparison between approaches on a large-scale setup.
Also, \zexe could be used to perform security evaluations of commercial products, and we envision future analysis to consider real-world antivirus programs as well~\cite{demetrio2021functionality}. 

\section*{Acknowledgments}
This work was partially supported by projects SERICS (PE00000014) and FAIR (PE00000013) under the NRRP MUR program funded by the EU - NGEU;
and by European Research Council (grant SLING 819789); 
and by the European Commission (Horizon Europe grant ELIAS 101120237);
and by the US Air Force Office of Scientific Research (FA8655-22-1-7034);
and by the Ministry of Education, University and Research (FARE grant ML4IP R205T7J2KP);
and by a France 2030 support managed by the Agence Nationale de la Recherche, under the reference ANR-23-PEIA-0004 (PDE-AI project);
and by ``InfoAICert''  funded by the MUR program “Fondo italiano per le scienze applicate” (FISA-2023-00128).
This work represents only the view of the authors.
The European Commission and the other organizations are not responsible for any use that may be made of the information it contains.

\bibliographystyle{unsrt}
\bibliography{references}

\appendix

\section{Auxiliary Results} \label{app:aux_results}
We report here every lemmas and propositions required to prove the theoretical results of \autoref{sec:theoretical}. 

\mypar{Notation.} From now on, we denote with $\mathcal{F}_k$ the filtration $ \sigma(G_0, \cdots, G_{k - 1})$. 
We denote with $F_z$ the target function of \autoref{eqn:adv_problem} for a malware $z \in Z$ and $F_h$ as its smooth approximation defined as
\begin{equation}\label{eqn:smoothing_fz}
    F_h(x) := \mathbb{E}_{u \in \mathbb{B}^d}[F_z(x + hu)],
\end{equation}
where $\mathbb{B}^d:=\{x\in\mathbb{R}^d \, | \, \|x\| \leq 1\}$ is the $d$ dimensional unit-ball. 
To improve the readability of lemmas and proofs, we omit the indication of $z$ in $F_h$.

\begin{proposition}[Smoothing properties]\label{prop:smt_props}
 Let $F_h$ be the smooth approximation defined in \autoref{eqn:smoothing_fz}. Then, if $F_z$ is $L$-Lipschitz continuous - i.e., $\forall x, y\in \mathbb{R}^{d}$, $|F_z(x) - F_z(y)| \leq L \| x - y \|$, $F_h$ is $L$-Lipschitz continuous, differentiable and for every $x,y \in \mathbb{R}^d$
 \begin{equation*}
    \begin{aligned}
         \| \nabla F_h(x) - \nabla F_h(y) \| \leq \frac{L \sqrt{d}}{h} \| x - y \|,%
    \end{aligned}
 \end{equation*}
 and
 \begin{equation*}
     F_h(x) \leq F_z(x) + L h.
 \end{equation*}
\end{proposition}
\begin{proof}
This is a standard result proposed and proved in different works - see for example 
\cite[Lemma 8]{duchi_rand_smoothing},\cite[Proposition 2.2]{NEURIPS2022_a78f142a} and \cite{yousefian2012stochastic}. 
\end{proof}

\begin{lemma}[Approximation Error]\label{lem:approx_error}
    Let $g(\cdot)$ be defined as in \autoref{eqn:fw_struct_fd} for arbitrary $h >0$ and $G \in O(d)$. Then, if $F_z$ is $L$-Lipschitz (see \autoref{asm:lip_cont}), for every $v \in \mathbb{R}^d$ 
    \begin{equation*}
        \| g(v) \|^2 \leq \frac{d^2 L^2}{\ell}.
    \end{equation*}
\end{lemma}
\begin{proof}
    By \autoref{eqn:fw_struct_fd}, we have
    \begin{equation*}
        \| g(v) \|^2 = \frac{d^2}{\ell^2 h^2} \Bigg\| \sum\limits_{i = 1}^\ell (F_z(v + h G e_i) - F_z(v))G e_i \Bigg\|^2.
    \end{equation*}
    By orthogonality and since $\| G e_i \| = 1$,
    \begin{equation*}
        \| g(v) \|^2 = \frac{d^2}{\ell^2 h^2} \sum\limits_{i = 1}^\ell (F_z(v + h G e_i) - F_z(v))^2.
    \end{equation*}
    By Assumption \ref{asm:lip_cont}, we get the claim.
\end{proof}

\begin{lemma}[Smoothing Lemma]\label{lem:smoothing_lemma}
  Given a probability space $(\Omega, \mathcal{F}, \mathbb{P})$, let $G: \Omega \to O(d)$ be a random variable where $O(d)$ is the orthogonal group endowed with the Borel $\sigma$-algebra. Assume that the probability distribution of $G$ is the (normalized) Haar measure.  Let $h > 0$ and let  $g$ be the finite difference approximation defined in \autoref{eqn:fw_struct_fd}. Then, let $F_h$ be the smooth approximation of the objective function $F_z$ defined in \autoref{eqn:smoothing_fz}
     \begin{equation*}
      (\forall v\in\mathbb{R}^d)\qquad   \mathbb{E}_G[g_{(G, h)}(v)] = \nabla F_{h}(v),
     \end{equation*}
\end{lemma}
\begin{proof}
    The proof follows the same line of \cite[Lemma 1]{rando2024optimal}.
\end{proof}

\begin{lemma}[Function value decrease]\label{lem:fun_values_decrease}
    Under \autoref{asm:lip_cont}    
    \begin{align*}
        \mathbb{E}_{G_k}[F_{h}(v_{k + 1}) \, | \, \mathcal{F}_k] \\ \leq F_{h}(v_k) - \gamma_k \| \nabla F_{h}(v_k) \|^2 + \frac{L \sqrt{d}}{h} \gamma_k^2 \mathbb{E}_{G_k}[\| g_k(v_{k}) \|^2 \, | \, \mathcal{F}_k].
    \end{align*}
\end{lemma}
\begin{proof}
In the following, we denote with $g_k(\cdot) := g_{(G_k, h)}(\cdot)$. By Proposition \ref{prop:smt_props}, we have that $F_{h}$ is $L \sqrt{d} / h$ smooth. Thus, by the Descent Lemma \cite{polyak1987introduction},
\begin{align*}
    F_{h}(v_{k + 1}) \\ \leq F_{h}(v_k) - \gamma_k \scalarprod{\nabla F_{h}(v_k)}{g_k(v_k)} + \frac{L \sqrt{d}}{h} \gamma_k^2 \| g_k(v_k) \|^2.
\end{align*}
Taking the conditional expectation, by \autoref{lem:smoothing_lemma} we get the claim.
\begin{align*}
    \mathbb{E}_{G_k}[F_{h}(v_{k + 1}) \, | \, \mathcal{F}_k] \\ \leq F_{h}(v_k) - \gamma_k \| \nabla F_{h}(v_k) \|^2 + \frac{L \sqrt{d}}{h} \gamma_k^2 \mathbb{E}_{G_k}[\| g_k(v_{k}) \|^2 \, | \, \mathcal{F}_k].
\end{align*}
\end{proof}

\section{Proofs of Theoretical Results} \label{app:proof_main_thm}

\subsection{Proof of \autoref{thm:nonconv_eps_pos}}
    By \autoref{lem:fun_values_decrease} and \autoref{lem:approx_error}, we get for $k = \underline{K}, \cdots, \bar{K}$
    \begin{equation*}
        \mathbb{E}[F_{h}(v_{k + 1}) \, | \, \mathcal{F}_k] \leq F_{h}(v_k) - \gamma_k \| \nabla F_{h}(v_k) \|^2 + \frac{L^3 d^2 \sqrt{d}}{ \ell } \frac{\gamma_k^2}{h}.
    \end{equation*}
    Taking the full expectation,
    \begin{equation*}
        \gamma_k \mathbb{E}[\| \nabla F_{h}(v_k) \|^2] \leq \mathbb{E}[F_{h}(v_k) - F_{h}(v_{k + 1})] + \frac{L^3 d^2 \sqrt{d}}{ \ell} \frac{\gamma_k^2}{h}.
    \end{equation*}
    Summing from $i= \underline{K}, \cdots, \bar{K}$,
    \begin{align*}
        \sum\limits_{i = \underline{K} }^{\bar{K}} \gamma_i \mathbb{E}[\| \nabla F_{h}(v_i) \|^2] \\ \leq \mathbb{E}[F_{h_{\underline{K}}}(v_{\underline{K}}) - F_{h}(v_{k + 1})] + \frac{L^3 d^2 \sqrt{d}}{ \ell} \sum\limits_{i = \underline{K}}^{\bar{K}} \frac{\gamma_i^2}{h}.
    \end{align*}
    By the definition of $F_{h}$, we have $F_{h}(v) \geq \min F_z$ for every $v \in \mathbb{R}^d$ and $h > 0$. Thus,
    \begin{equation*}
        \sum\limits_{i = \underline{K}}^{\bar{K}} \gamma_i \mathbb{E}[\| \nabla F_{h}(v_i) \|^2] \leq F_{h}(v_{\underline{K}}) - \min F_z + \frac{L^3 d^2 \sqrt{d}}{ \ell} \sum\limits_{i = \underline{K}}^{\bar{K}} \frac{\gamma_i^2}{h}.
    \end{equation*}
    This concludes the proof.

\subsection{Proof of \autoref{cor:nonconv_eps_pos} }
By \autoref{thm:nonconv_eps_pos}, we have
\begin{equation*}
    \eta_{\underline{K}}^{\bar{K}} \leq \Bigg( F_{h_{\underline{K}}}(v_{\underline{K}}) - \min F_z + \frac{L^3 d^2 \sqrt{d}}{ \ell} \sum\limits_{i = \underline{K}}^{\bar{K}} \frac{\gamma_i^2}{h} \Bigg) / \Bigg( \sum\limits_{\underline{K}}^{\bar{K}} \gamma_i \Bigg).
\end{equation*}
Due to the choice of $\gamma_k = \gamma$, we get
\begin{equation*}
    \eta_{\underline{K}}^{\bar{K}} \leq \frac{F_{h}(v_{\underline{K}}) - \min F_z}{\gamma (\bar{K} - \underline{K})} + \frac{L^3 d^2 \sqrt{d}}{ h \ell (\bar{K} - \underline{K})} \gamma.
\end{equation*}
Now, let $K = \bar{K} - \underline{K}$. Minimizing the right handsize with respect to $\gamma$ we get
\begin{equation*}
    \hat{\gamma} = \sqrt{\frac{(F_h(v_{\underline{K}}) - \min F_z )\ell h }{KL^3d^2 \sqrt{d}}}.
\end{equation*}
Let $\varepsilon \in (0,1)$. Choosing $\gamma = \hat{\gamma}$ we get $\eta_{\underline{K}}^{\bar{K}} \leq \varepsilon$ if 
\begin{equation*}
    K \geq 4\frac{(F_h(v_{\underline{K}}) - \min F_z )L^3 d^2 \sqrt{d} }{\ell h} \epsilon^{-2}.
\end{equation*}
This concludes the proof.

\end{document}